\documentclass[nohyperref]{article}

\usepackage{microtype}
\usepackage{graphicx}
\usepackage{subfigure}
\usepackage{booktabs} %
\usepackage{xspace}

\usepackage{hyperref}

\usepackage{enumitem}
\setlist[itemize]{itemsep=0.01cm}
\setlist[enumerate]{itemsep=0.01cm}

\usepackage{color}

\usepackage[accepted]{icml2022preprint}

\usepackage{amsmath}
\usepackage{amssymb}
\usepackage{mathtools}
\usepackage{amsthm}
\usepackage{bbm}

\usepackage[capitalize,noabbrev]{cleveref}

\theoremstyle{plain}
\newtheorem{theorem}{Theorem}[section]
\newtheorem{proposition}[theorem]{Proposition}

\theoremstyle{definition}

\theoremstyle{remark}

\usepackage[textsize=tiny]{todonotes}

\icmltitlerunning{RL-DAD\xspace: Sequential Bayesian Experimental Design for Non-Differentiable Implicit Models}
\DeclareMathOperator*{\argmax}{arg\,max}

\DeclareMathOperator{\ent}{H}

\begin{document}

\twocolumn[
\icmltitle{Policy-Based Bayesian Experimental Design for \\ Non-Differentiable Implicit Models}

\icmlsetsymbol{equal}{*}

\begin{icmlauthorlist}
\icmlauthor{Vincent Lim}{ucb}
\icmlauthor{Ellen Novoseller}{ucb}
\icmlauthor{Jeffrey Ichnowski}{ucb}
\icmlauthor{Huang Huang}{ucb}
\icmlauthor{Ken Goldberg}{ucb}

\end{icmlauthorlist}

\icmlaffiliation{ucb}{BAIR, UC Berkeley, Berkeley CA, USA}

\icmlcorrespondingauthor{Vincent Lim}{vincentklim@berkeley.edu}

\icmlkeywords{Machine Learning, ICML}

\vskip 0.3in
]

\printAffiliationsAndNotice{\icmlEqualContribution} %

\newcommand{\rloedlong}{Reinforcement Learning for Deep Adaptive Design\xspace}
\newcommand{\rloed}{RL-DAD\xspace}
\newcommand{\Vincent}[1]{{\color{red}[#1 -VL]}}
\newcommand{\Ellen}[1]{{\color{blue}[#1 -EN]}}
\newcommand{\Jeff}[1]{{\color{purple}[#1 -JI]}}
\newcommand{\mcA}{\mathcal{A}}
\newcommand{\mcS}{\mathcal{S}}
\newcommand{\mcP}{\mathcal{P}}

\newcommand{\para}[1]{{\textbf{#1} \hspace{3mm}}}

\begin{abstract}
For applications in healthcare, physics, energy, robotics, and many other fields, designing maximally informative experiments is valuable, particularly when experiments are expensive, time-consuming, or pose safety hazards. While existing approaches can sequentially design experiments based on prior observation history, many of these methods do not extend to implicit models, where simulation is possible but computing the likelihood is intractable. Furthermore, they often require either significant online computation during deployment or a differentiable simulation system. We introduce \rloedlong (\rloed), a method for simulation-based optimal experimental design for non-differentiable implicit models.
\rloed extends prior work in policy-based Bayesian Optimal Experimental Design (BOED) by reformulating it as a Markov Decision Process with a reward function based on likelihood-free information lower bounds, which is used to learn a policy via deep reinforcement learning. The learned design policy maps prior histories to experiment designs offline and can be quickly deployed during online execution.
We evaluate \rloed and find that it performs competitively with baselines on three benchmarks.  %

\end{abstract}
\section{Introduction}
Designing experiments to efficiently gather data is a key issue across scientific and engineering disciplines. Some experiments, especially physical experiments, can be expensive, time-consuming, or risky~\cite{lim2021planar, li2021roial}. %
Meanwhile, well-designed experiments can accelerate simulation tuning~\cite{lim2021planar}; benefit learning for robotics~\cite{zhao2021mutual}, for instance improving robot Sim2Real transfer~\cite{lim2021planar}; and improve perception models~\cite{huang2020mechanical}. Bayesian Optimal Experimental Design (BOED)~\cite{10.1214/aoms/1177728069} is one framework for designing experiments to maximize the information that observations yield about unknown parameters of interest.
Prior work has shown the potential of optimal experiment design in domains such as MRI image reconstruction~\cite{yin2021end}, protein engineering~\cite{wittmann2020machine}, and photonic nanostructure design~\cite{song2020mirrored}.

BOED seeks to design experiments to learn maximal information about the unknown parameters of a model.
This work considers \textit{sequential} BOED, in which each experiment is selected based on data from previous experiments.
Conventionally, sequential BOED algorithms select experiment designs that maximize a mutual information estimate
~\cite{kleinegesse2019efficient, kleinegesse2021sequential, foster2019variational}.
However, this process is myopic---only maximizing the information gain with respect to the next experiment---and is computationally impractical for online learning, as it requires an expensive mutual information maximization between each experiment.
Such methods often perform expensive online posterior updates~\cite{kleinegesse2019efficient} as optimizing the mutual information is doubly intractable~\cite{rainforth2018nesting}. 

This paper builds on recent work in policy-based BOED~\cite{foster2021deep}, which instead of performing expensive information gain maximizations during deployment, leverages simulation during training to learn a neural network-parameterized \textit{design policy} upfront. This policy maps an experiment history to the next design in a single forward pass, allowing for efficient online design decisions.
While~\citet{foster2021deep} focus on models with analytic likelihoods,~\citet{ivanova2021implicit} extend this work to implicit models (which do not assume closed-form likelihoods) with differentiable simulators. 

In many real-world domains, including robotics, medicine, particle physics, and protein folding, we cannot assume simulator differentiability. This work builds on prior work on policy-based BOED by proposing \emph{\rloedlong (\rloed)}, a general algorithm for sequential BOED for \emph{non-differentiable} implicit models that only assumes access to a black box simulator. \rloed frames BOED as a Partially Observed Markov Decision Process (POMDP) and uses reinforcement learning (RL) to select experiment designs %
and optimize the likelihood-free expected information-gain estimator derived in \citet{ivanova2021implicit}. 
In contrast to prior work on policy-based BOED, which selects experiments by directly optimizing expected information-gain estimates, we propose an RL reward function %
that incentivizes the RL agent to maximize the expected information gain. Learning a design policy via RL maintains the benefits of policy-based BOED---selecting designs non-myopically~\cite{sutton2018reinforcement} and requiring only network forward passes during deployment---while relaxing assumptions from prior work.

In experiments, we evaluate \rloed on three benchmarks and show that it approaches the performance of prior policy-based BOED methods that require stricter assumptions, namely access to 1) an analytic likelihood or 2) a differentiable simulator. In addition, results demonstrate that \rloed yields promising performance on problems that lack closed-form likelihoods and simulator differentiability. This paper contributes:
\begin{itemize}
    \item A formulation of policy-based BOED as a POMDP with a theoretically-justified dense reward signal.
    \item \rloed, an RL-based algorithm for sequential policy-based BOED.
    \item A policy network architecture suitable for BOED.
    \item Experiments showing that \rloed approaches the performance of prior work that requires stricter assumptions, particularly in high dimensions, while also extending to models for which prior policy-based BOED methods do not apply.
\end{itemize}

\section{Related Work}\label{sec:related-work}

Bayesian optimal experiment design (BOED)~\cite{10.1214/aoms/1177728069} aims to select experiments that yield maximal information about unknown model parameters. 
This work proposes a sequential BOED approach compatible with implicit models---which do not require a closed-form likelihood for observations---and non-differentiable simulators.

\para{Information gain estimation} Many BOED algorithms select experiment designs by maximizing estimates of the expected information gain with respect to the unknown model parameters. 
Maximizing the expected information gain is equivalent to maximizing the information-theoretic mutual information~\cite{cover2012elements} between the model parameters and experiment outcome. Prior work on BOED leverages various mutual information estimation techniques, including variational estimators such as MINE~\cite{pmlr-v80-belghazi18a} and SMILE~\cite{song2020understanding}, as well as contrastive sampling-based methods such as sPCE~\cite{foster2021deep} and InfoNCE~\cite{ivanova2021implicit}.

\para{Static BOED} Static BOED selects a set of experiments upfront. \citet{pmlr-v130-zhang21l} perform static BOED for likelihood-based models using the SMILE mutual information estimator~\cite{song2020understanding}, and \citet{pmlr-v119-kleinegesse20a} consider static BOED for implicit models by optimizing the MINE estimator~\cite{pmlr-v80-belghazi18a}. For implicit models, these estimators are typically optimized via evolutionary methods~\cite{beyer2006evolution} or Bayesian Optimization~\cite{snoek2012practical}

\para{Sequential BOED} Sequential BOED algorithms select experiment designs based on existing experiment outcomes. Many such algorithms, e.g.~\citet{watson2017quest+, foster2019variational, kleinegesse2019efficient, kleinegesse2021sequential}, alternate between 1) estimating information gain values 
and 2) optimizing the information gain to select the next design. For instance,~\citet{kleinegesse2019efficient, kleinegesse2021sequential} handle implicit models by estimating the log-density ratio between the prior and posterior, while~\citet{foster2020unified} propose a gradient-based approach that jointly estimates the information gain and optimizes the experiment design. Critically, these methods are myopic, in that each experiment design is chosen via greedy information gain optimization rather than by multi-step planning. In addition, these methods perform time-intensive mutual information estimation between each experiment. Policy-based BOED, discussed next, addresses both of these limitations.

\para{Policy-based BOED} Several recent sequential BOED works~\cite{huan2016sequential, foster2021deep, ivanova2021implicit} learn non-myopic \textit{design policies} for selecting experiments. These approaches amortize the cost of sequential BOED by training a policy upfront that maps an experiment history to a subsequent experiment design. While training this policy requires significant upfront compute, at deployment time, the trained policy can efficiently output experiment designs rather than optimizing an objective online between experiments. 

\citet{huan2016sequential} formulate BOED as a Markov decision process and leverage dynamic programming to obtain a policy mapping belief states to experiment designs. However, dynamic programming methods do not scale well to large state spaces. Meanwhile,~\citet{foster2021deep} introduce the DAD algorithm, which optimizes policy parameters via gradient ascent to maximize the information gain.
While these works require closed-form likelihoods,~\citet{ivanova2021implicit} develop the iDAD algorithm, which utilizes an optimization objective for implicit models; however, the method requires a differentiable simulator. %

\para{RL in BOED} This work leverages RL to perform policy learning in sequential BOED. This contrasts with prior policy-based BOED methods, which utilize dynamic programming~\cite{huan2016sequential} and gradient ascent~\cite{foster2021deep, ivanova2021implicit} rather than RL. ~\citet{shen2021bayesian, blau2022optimizing, ashenafi2021reinforcement} consider RL-based sequential BOED, but assume closed-form likelihoods; furthermore, the method in~\citet{ashenafi2021reinforcement} is restricted to the Bayesian function optimization setting~\cite{snoek2012practical}, rather than general BOED.~\citet{ivanova2021implicit} introduce a policy-based method for implicit models that requires a differentiable simulator. In contrast, our RL-based method applies to the general BOED setting and does not assume access to either a closed-form likelihood or a differentiable simulator.

\section{Preliminaries}\label{sec:background}

We consider the experiment design problem setting, in which a learning agent seeks to maximize its knowledge about a set of initially-unknown model parameters. The agent sequentially performs experiments and observes their outcomes to learn about the model parameters. The agent aims to gain maximal information by making optimal choices about which experiments to perform.

\subsection{The Sequential BOED Problem Statement}\label{ssec:background-boed}

This work considers sequential BOED, in which a learning agent seeks to learn about an unknown set of model parameters $\theta \in \Theta$, where $\Theta$ is the space of possible model parameter values. The agent sequentially performs experiments, where each experiment is defined by an \textit{experiment design} that specifies the parameters of that experiment. On each iteration $t \in \{1, \ldots, T\}$, for a fixed, given time horizon $T$, the agent selects an experiment design $\xi_t \in \Xi$ and receives an observation $y_t \in Y$, where $\Xi$ and $Y$ are the design and observation spaces, respectively.

At time $t + 1$, the agent has access to the current history of design-observation pairs $h_t:=\{(\xi_1, y_1), \ldots, (\xi_t, y_t)\} \in \mathcal{H}^t$, where $\mathcal{H}^t = (\Xi \times Y)^t$ is the space of possible length-$t$ histories. Thus, the agent can utilize its current information about $\theta$ to adaptively design the next experiment.

\para{Learning objective} The goal of this work is to select designs $\xi_1, \ldots, \xi_T$ that maximize the mutual information between the observations $y_1, \ldots, y_T$ and $\theta$:
\begin{equation}\label{eqn:IG_objective}
    I(\theta; y_1, \ldots y_T \mid \xi_1, \dots \xi_T).
\end{equation}
We further aim to select designs with high computational efficiency, such that our method that can quickly be deployed in live experiments.

\para{Assumptions} We make the following assumptions:
\begin{enumerate}
    \item The design, observation, and parameter spaces, $\Xi, Y$, and $\Theta$, are known.
    \item Relating $\xi_t$, $y_t$, and $\theta$, we assume access to a simulation model $\mathcal{M}$ that maps an experiment design, history, and set of parameters to an observation via a possibly-unknown likelihood: $y_t \sim p(y \mid \theta, \xi_t, h_{t - 1})$.
\end{enumerate}

This work makes fewer assumptions about the simulator than a number of prior works. Firstly, we do not assume simulator differentiability, as in~\citet{ivanova2021implicit}. Secondly, we do not assume that experiments must be conditionally independent, where conditionally independent experiments have likelihoods of the form, $y_t \sim p(y \mid \theta, \xi_t)$. Instead, we assume that the likelihood of observation $y_t$ is generated via $p(y \mid \xi_t, \theta, h_{t - 1})$, where the $h_{t - 1}$-dependence allows the system's underlying state to depend on the specific experiment sequence performed. Experiments can be conditionally-dependent in many real-world settings, for instance in robotics and biology.

\subsection{Background: Sequential BOED with design policies}

In settings with implicit models, in which the likelihood $p(y \mid \xi, \theta, h_{t - 1})$ is unknown, sequential BOED conventionally places a prior $p(\theta)$ over the unknown parameters $\theta$ and then alternates between two steps: 1) estimating the intermediate posterior $p(\theta \mid h_{t-1})$ via likelihood-free inference,  and 2) maximizing the expected marginal information gain with respect to $\theta$, defined as the expected reduction in entropy from the current intermediate posterior to the next intermediate posterior:
\begin{flalign}
        &I_{h_{t-1}}(\xi_t) := I(y_t; \theta \mid h_{t - 1}, \xi_t) \label{eqn:information-gain} \\&\hspace{3mm}= \ent[p(\theta\mid h_{t-1})]
        -\mathbb{E}_{p(y_t\mid\xi_t)}\ent[p(\theta\mid h_{t-1} \cup \{\xi_t, y_t\})]]. \nonumber
\end{flalign}

In this setting, we define $h_0:=\emptyset$, such that $p(\theta|h_0)=p(\theta)$. This process is myopic, in that it maximizes the single-step EIG without consideration for future steps, which can lead to suboptimal design sequences \cite{pmlr-v119-jiang20b}. 

Notably, the mutual information objective in Equation~\eqref{eqn:IG_objective} can be decomposed as the sum of expected marginal information gains, as shown in Appendix~\ref{app:EIG_decomp}: %
\begin{equation}\label{eqn:sum_marginals}
    I(\theta; y_1, \ldots, y_T \mid \xi_1, \ldots, \xi_T) = \sum_{t = 1}^T I_{h_{t - 1}}(\xi_t).
\end{equation}

Meanwhile, recent works in sequential BOED for models with tractable likelihoods~\cite{foster2021deep} and for implicit models with differentiable simulators~\cite{ivanova2021implicit} instead learn a design policy $\pi_\phi$ with parameters $\phi$. At time $t$, this policy is a deterministic function of the history $h_{t-1}$ that outputs the next design $\xi_t$. In particular, \citet{foster2021deep} seek to identify the policy $\pi_\phi$ that maximizes the total expected information gain (EIG) for a policy:
$$\mathcal{I}_T(\pi_\phi) = \mathbb{E}_{p(\theta)p(h_T\mid\theta, \pi_\phi)}\left[\sum_{t=1}^T I_{h_{t-1}}(\xi_t)\right], \xi_t = \pi_\phi(h_{t-1}),$$
where the policy $\pi_\phi$ is deterministic. While our work instead considers stochastic policies, for which $\xi_t \sim \pi_\phi(h_{t-1})$, we adapt this result for stochastic policies in Appendix~\ref{app:total_EIG}.

\citet{foster2021deep} derive the sequential Prior Contrastive Estimation (sPCE) optimizable lower bound to $\mathcal{I}_T(\pi_\phi)$:
\begin{flalign}
    \mathcal{L}&_T^{\rm sPCE}(\pi_\phi, L)={}\label{eqn:spce}\\
        &\mathbb{E}_{p(\theta_{0:L})p(h_T|\theta_0, \pi_\phi)}\left[\log\frac{p(h_T|\theta_0, \pi_\phi)}{\frac{1}{L+1}\sum_{\ell=0}^L p(h_T|\theta_\ell, \pi_\phi)}\right]. \nonumber
\end{flalign}

However, utilizing (\ref{eqn:spce}) to optimize the objective $\mathcal{I}_T(\pi_\phi)$ requires a tractable likelihood. Thus, more recently,~\citet{ivanova2021implicit} extend the sPCE bound to the InfoNCE bound for implicit models by using a jointly optimized auxiliary critic $U:\mathcal{H}^T \times \Theta \mapsto \mathbb{R}$ that approximates the log-likelihood. The InfoNCE bound is defined by:
\begin{flalign}
    \mathcal{L}&_T^{\rm NCE}(\pi_\phi, U; L)={}\label{eqn:nce}\\
        &\mathbb{E}_{p(\theta_{0:L})p(h_T|\theta_0, \pi_\phi)}\left[\log\frac{\exp{U(h_T, \theta_0)}}{\frac{1}{L+1}\sum_{\ell=0}^L \exp{U(h_T, \theta_\ell)}}\right]. \nonumber
\end{flalign}

Notably, the optimal critic (which assumes $U$ has infinite approximation capacity) is $U^*_{\rm NCE}=\log p(h_T|\theta_0, \pi_\phi) + c(h_T)$, where $c(h_T)$ is any function depending only on the history $h_T$~\cite{ivanova2021implicit}. With an optimal critic $U^*$, we recover the sPCE bound from the InfoNCE bound: $\mathcal{L}_T^{\rm NCE}(\pi_\phi, U^*_{\rm NCE}; L)=\mathcal{L}_T^{\rm sPCE}(\pi_\phi, L)$. With a learned critic $U$, we are able to obtain a point-wise posterior estimator via self-normalization~\cite{ivanova2021implicit}.

During training of the policy network $\pi_\phi$, the policy parameters $\phi$ can be optimized by maximizing either of the lower bounds in Equations~\eqref{eqn:spce} and \eqref{eqn:nce} via stochastic gradient ascent on simulated histories. However, each lower bound has a restrictive set of assumptions: the sPCE bound requires models with a closed-form likelihood density $p(h_T\mid\theta, \pi)$, while the InfoNCE bound requires that the simulation model $y \sim p(y\mid\theta, \xi)$ be differentiable. We present a method that relaxes these assumptions to the nondifferentiable implicit case, where the likelihood density is unknown or intractable and the simulation model is nondifferentiable. 

\subsection{Background: Reinforcement Learning}\label{ssec:background-rl}
We consider the standard reinforcement learning (RL) paradigm formulated as a episodic %
Partially-Observed Markov Decision Process (POMDP), specified by the tuple $(\mcS, \mathcal{O}, \mcA,  \mcP, \mathcal{P}_e, r, \gamma, T)$ with states $s \in \mcS$, observations $o \in \mathcal{O}$, actions $a \in \mcA$, transition dynamics $\mcP: \mcS\times\mcA\times\mcS\mapsto [0, 1]$, observation emission probabilities $\mathcal{P}_e: \mathcal{S} \times \mathcal{O} \mapsto [0, 1]$, reward function $r:\mcS\times\mcA\mapsto\mathbb{R}$, discount factor $\gamma$, and fixed time horizon $T$. Over a series of discrete timesteps $t$, the agent receives an observation $o$, selects an action $a$, and receives a reward $r$ and next observation $o'$. The agent interacts with the environment through trajectories of the form $\tau = \{s_0, a_0, r_1, s_1, \ldots, s_{T - 1}, a_{T - 1}, s_T, r_T\}$, where the true states $s_i$ are not observed by the learning agent. The cumulative reward is the discounted sum of rewards $R=\sum_{i=0}^T \gamma^{i}r(s_i, a_i)$ with discount factor $\gamma$ over the time horizon $T$. The goal of the agent is to find the optimal parameters $\phi$ of a policy $\pi_\phi:\mathcal{O}\mapsto\mcA$ that maximize the expected return:
$$\phi^*=\argmax_\phi\mathbb{E}_{p(\tau|\pi_\phi)}\left[\sum_{i=0}^{T - 1} \gamma^{i}r(s_i, a_i)\right],$$
where $p(\tau|\pi_\phi)$ is the trajectory distribution of the policy.

This work leverages TD3 \cite{fujimoto2018addressing}, a deep off-policy policy gradient algorithm that concurrently learns a Q-function via Bellman backups and leverages it to learn a policy. Both the policy and Q-function are parameterized by neural networks. %
During training, zero-mean Gaussian noise is added to the policy outputs to improve exploration.

\section{The \rloed Algorithm}\label{sec:rloed}
The key idea of \rloed is to cast the conventional BOED framework as an MDP, in which the RL agent learns a policy that maps the observed history  of experiment designs and outcomes to the next design. First, we define a naive sparse reward formulation, in which the agent receives zero reward until the final timestep of each length-$T$ trajectory of designs and experiment outcomes; then, in the last time-step, the agent is rewarded based on the total information gained about $\theta$ during the interaction trajectory. As RL algorithms are often inefficient with sparse reward signals ~\cite{pathak2017curiosity,NIPS2017_453fadbd}, we subsequently extend this reward definition to obtain a dense reward signal in Section \ref{ssec:dense-reward}, which we show helps to accelerate learning. The pseudocode of \rloed is detailed in Algorithm~\ref{alg:rloed}.

\begin{algorithm}[tb]
   \caption{\rloed with dense rewards} \label{alg:rloed}
\begin{algorithmic}[1]
   \STATE {\bfseries Requires:} simulator $\mathcal{M}$, prior $(\theta)$, batch size $L$
   \STATE Initialize replay buffer $\mathcal{B} \leftarrow \emptyset$
   \STATE Initialize policy and critic parameters $\pi_\phi$ and $U_\psi$
   \STATE // Simulate history for policy and critic learning
   \REPEAT
   \STATE // Sample a batch of parameters from the prior
   \STATE $\theta_{0:L} \sim p(\theta)$
   \STATE Initialize empty histories $h^{0:L}_0 \leftarrow \emptyset$
   \FOR{$t=1$ {\bfseries to} $T$}
   \STATE // Get next action, simulate and add to the history
   \STATE $\xi_t^{0:L} \leftarrow \pi_\phi(h_{t-1}^{0:L})$
   \STATE $y_t^{0:L} \sim \mathcal{M}(\xi_t^{0:L})$
   \STATE $h_t^{0:L} \leftarrow h_{t-1}^{0:L} \cup \{\xi_t, y_t\}^{0:L}$
   
   \STATE // Calculate the dense reward
   \STATE $r_t^{0:L} \leftarrow g(h_t^{0:L}, U_\psi; L) - g(h_{t-1}^{0:L}, U_\psi; L)$
   \STATE // Store transitions in replay buffer
   \STATE $\mathcal{B} \leftarrow \mathcal{B} \cup \{(h_{t-1}^{0:L}, \xi_t^{0:L}, h_{t}^{0:L}, r_t^{0:L})\}$
   \ENDFOR
   \FOR{each gradient step}
   \STATE // Optimize policy via TD3
   \STATE $\phi \leftarrow \phi + \nabla J(\pi_\phi)$
   \ENDFOR
   \FOR{each gradient step}
   \STATE // Optimize critic using InfoNCE loss
   \STATE Sample minibatch ${(\tau_j)}_{j=0}^L$ from previous rollout
   \STATE Optimize $\mathcal{L}_T^{\rm NCE}$ with respect to $U_\psi$
   \ENDFOR
   \UNTIL{convergence}
\end{algorithmic}
\end{algorithm}

\subsection{Casting BOED as an MDP}\label{ssec:boed-mdp}
At each step $t$ of learning, the RL agent observes a history $h_t \in \mathcal{H}^t$, selects the next experiment design, and receives an updated history and a reward based on the InfoNCE objective. Below, we formalize the components of the POMDP.

\para{State and Observation Space} At each step $t$, the agent observes the history $h_t \in \mathcal{H}^t$ for $t \in \{0, \ldots, T\}$. Thus, the observation space $\mathcal{O}$ is given by the union of possible history spaces at all time-steps: $\mathcal{H}^0 \cup \mathcal{H}^1 \cup \ldots \mathcal{H}^T$. As the history is variable length, we consider several different history aggregation architectures in Section~\ref{ssec:architecture}, which encode the history into a fixed dimension latent vector. This fixed dimension vector functions as the state input to the policy $\pi(s)$ and Q-function $Q(s, a)$. Intuitively, the agent is learning to represent the observation $h_t$ as a current belief about the model posterior of $\theta_0$. However, we don't explicitly compute the model posterior over $\theta_0$, as this is computationally intractable for implicit models.

We define the state space as a tuple $s_t:=(h_t, \theta_{0:L})$. This consists of a ground truth parameter $\theta_0$, which the agent is trying to learn; the current experiment history $h_t$;
and $L$ contrastive samples of the parameters $\theta_0$, $\theta_{1:L}$. In addition to the experiment history $h_t$, the state includes $\theta_{0:L}$ because the reward signal, defined below, depends on the $\theta_{0:L}$ variables, which differ between trajectories; in fact, the ground truth $\theta_0$ and contrastive samples $\theta_{1:L}$ are independently sampled from the prior $p(\theta)$ and fixed during each trajectory. 
However, the RL agent does not observe any of the parameter vectors $\theta_{0:L}$. %

\para{Action Space} Because sequential BOED designs experiments adaptively based on previous history, we correspondingly seek a policy $\pi_\phi$ that maps a representation of the experiment history to the next experiment $\xi$. The policy outputs the next experiment design, and so the action space is given by the known space of possible designs $\Xi$. 

\para{Transition and Observation Emission Probabilities} As the parameter samples $\theta_{0:L}$ are fixed throughout within a length-$T$ trajectory, the transition probabilities depend only on the probability of moving from $h_t$ to $h_{t+1}$ given an action $\xi_{t+1}$. As the next history $h_{t+1} := h_t \cup \{(\xi_{t+1}, y_{t+1})\}$ includes the prior history, the transition probabilities $\mcP$ only depend on the probability of observing $y_{t+1}$ under the ground truth $\theta_0$ and design $\xi_{t + 1}$, and are thus given by: $p(s_{t+1}\mid s_t, a_t)=p(y_{t+1}\mid \theta_0, \xi_{t+1}, h_t)$. %

Since the agent observes the history, which is fully defined by the state, the observation emission probabilities $\mathcal{P}_e$ for a given state $s_t$ are given by: $p(o_t\mid s=s_t) = \mathbbm{1}_{[o_t = h_t]}$, where $\mathbbm{1}$ is the indicator function. We prove in Appendix~\ref{app:markov} that these transition dynamics---as well as the subsequently-defined rewards---are Markovian.

\para{Rewards} %
We first define a sparse reward to incentivize the RL agent to seek experiments that produce a high information gain between the observations and parameters $\theta_0$; this sparse reward motivates our dense reward formulation in Section~\ref{ssec:dense-reward}. This sparse reward approximates the total information gain at the end of each trajectory, while assigning zero reward during intermediate steps within trajectories. To approximate a trajectory's information gain, we first lower-bound $\mathcal{I}_T(\pi_{\phi})$ by the sPCE bound in Equation~\eqref{eqn:spce}. Given a likelihood and samples of $\theta_{0:L}$ and $h_T$ given $\pi_\phi$:
\begin{flalign}
    r(s_T, a_T \mid \pi_\phi) &= r(h_T, \theta_{0:L}, \xi_{T + 1} \mid \pi_\phi) \nonumber \\ &= \log \left[\frac{p(h_T \mid \theta_0, \pi_\phi)}{\frac{1}{L + 1}\sum_{l = 0}^L p(h_T \mid \theta_l, \pi_\phi)} \right]. \label{eqn:sPCE_reward}
\end{flalign}

Yet, we cannot directly estimate this sPCE-based reward, as it requires a tractable likelihood. Thus, we further approximate the sPCE bound via the InfoNCE loss \cite{poole2019variational, oord2018representation} to define our RL reward signal.
For convenience, we define the following function $g$, which takes as input a history $h_t$, a set of contrastive samples $\theta_{1:L}$, and a learned critic $U_\psi: \left(\cup_{t = 0}^T\mathcal{H}^t\right) \times \Theta$ with parameters $\psi$, which approximates the unknown log-likelihood:
\begin{equation}
    g(h_t, U_\psi; L) = \log\left[\frac{\exp(U_\psi(h_t, \theta_0))}{\frac{1}{L+1}\sum_{\ell=0}^{L}\exp(U_\psi(h_t, \theta_\ell))}\right],
\end{equation}
where $g(h_0, U_\psi; L):=0$ by definition. Approximating the reward in~\eqref{eqn:sPCE_reward} via $g$, we provide the following reward signal to the RL agent at timestep $T$, given specific samples of $\theta_{0:L}$ and $h_T$ and the critic $U_\psi$:
\begin{equation}
    r(s_T, a_T) \approx g(h_T, U_\psi; L).
\end{equation}
Meanwhile, the RL agent receives zero reward for $t < T$.

By extending results from~\citet{ivanova2021implicit} to stochastic design policies $\pi_\phi$ (see Appendix~\ref{app:total_EIG}-\ref{app:NCE}), we note that $\mathbb{E}_{p(\theta_{0:L})p(h_T \mid \theta_0, \pi_\phi)}[g(h_T, U_\psi; L))] \leq I(\theta; h_T \mid \pi_\phi)$, i.e. in expectation, $g$ lower-bounds the mutual information between the trajectory history  and unknown parameters $\theta$ for a given policy $\pi_\phi$. %
We follow the history simulation process in \citet{foster2021deep}, which samples a ground truth model parameter $\theta_0$ from the prior $p(\theta)$ and iterates between querying the policy for the next design $\xi_t \sim \pi_\phi(h_{t-1})$ and simulating the chosen design $y_t = \mathcal{M}(\xi_t)$. 

By setting the RL discount factor to $\gamma=1$, the RL agent aims to maximize the following objective:
\begin{align*}
    J(\pi_\phi) &= \mathbb{E}_{p(\theta_{0:L})p(h_T\mid\theta_0, \pi_\phi)}\left[\sum_{i=0}^{T - 1} \gamma^{i}r(s_i, a_i)\right] \\
    &= \mathbb{E}_{p(\theta_{0:L})p(h_T\mid\theta_0, \pi_\phi)}\left[g(h_T, U_\psi; L)\right] \\
    &= \mathbb{E}_{p(\theta_{0:L})p(h_T\mid\theta_0, \pi_\phi)}\left[\frac{\exp(U_\psi(h_t, \theta_0))}{\frac{1}{L+1}\sum_{\ell=0}^{L}\exp(U_\psi(h_t, \theta_\ell))}\right].
\end{align*}

We see that under our POMDP construction, the RL agent is incentivized to optimize the objective in Equation~(\ref{eqn:nce}), as desired. The policy $\pi_\phi$ and auxiliary critic $U_\psi$ are jointly optimized via the InfoNCE loss for policy-based BOED (\ref{eqn:nce}). 

\subsection{Dense Reward Formulation}\label{ssec:dense-reward}

Since RL can struggle with sparse rewards~\cite{NIPS2017_453fadbd}, we additionally propose a dense reward formulation. While the sparse signal in Section~\ref{ssec:boed-mdp} assigns zero reward to intermediate timesteps ($t < T$) while assigning a reward of $g(h_T, U_\psi; L)$ in the final timestep, our proposed dense reward sets $r(h_{t-1}, \theta_{0:L}, \xi_t) := g(h_t, U_\psi; L)-g(h_{t-1}, U_\psi; L)$. Importantly, under the dense reward, the sum of rewards in a trajectory is:
\begin{equation*}
    \sum_{t=1}^T \left[g(h_t, U_\psi; L)-g(h_{t-1}, U_\psi; L)\right] = g(h_T, U_\psi; L),
\end{equation*}
where we collapse the telescoping sum and use that $g(h_0, U_\psi; L):=0$ by definition. Thus, the sparse and dense reward formulations assign each trajectory the same total reward; however, in Section \ref{sssec:ablation-sparse}, we note that the dense reward formulation provides significant empirical gains. Furthermore, we show in Appendix~\ref{app:dense_reward} that the dense rewards approximate the marginal expected information gains.

\subsection{Network Architectures}\label{ssec:architecture}

While the TD3 actor-critic architecture remains largely unchanged, both the policy network $\pi_{\phi}$ and Q-network $Q_{\omega}$ receive variable-length histories as inputs. Similarly, the critic $U_\psi$, which approximates the log-likelihood of a history, must also handle variable-length histories. 
 
Inspired by \citet{ivanova2021implicit}, the critic network $U_\psi$ has two inputs: a parameter vector $\theta$ and a history $h_t$. However, in contrast to \citet{ivanova2021implicit}, these inputted histories may be variable-length, rather than exclusively length-$T$. Despite this, we suggest following the same architecture in~\citet{ivanova2021implicit}: encode the history $h_t$ and parameter $\theta$ via two separate encoder networks $E_{\psi_h}(h_t)$ and $E_{\psi_\theta}(\theta)$, respectively, and then take the inner product of the two encodings, $E_{\psi_h}(h_t)^TE_{\psi_\theta}(\theta)$, to obtain a scalar output. %
The parameter encoding network $E_{\psi_\theta}$ is a simple feed-forward neural network, while the history encoding network depends on whether the experiments are conditionally independent. As in \citet{ivanova2021implicit}, we use a permutation-invariant sum-pooling architecture with an attention module for environments in which conditional independence does hold, while using an LSTM architecture otherwise.

The RL agent's policy network $\pi_{\phi}$ and Q-network $Q_{\omega}$ must also both take as input a variable-length history $h_t$. In contrast to \citet{ivanova2021implicit}---which suggests encoding the history in the same manner as the critic $U_\psi$'s history encoder $E_{\psi_h}(h_t)$---we suggest a simple concatenation architecture with zero-padding for the RL policy and Q networks. Empirically, we find that a simple concatenation architecture significantly outperforms a more complicated sequence-encoding architecture.

\section{Experiments}\label{experiments}

We evaluate \rloed on three experimental design environments, including learning model parameters in a location finding task (Section~\ref{ssec:locfin}), an epidemiology model (Section~\ref{ssec:sir}), and the cartpole environment (Section \ref{ssec:cartpole})

\para{Baselines} We compare \rloed to several baselines. In particular, we mainly consider baselines that do not require significant online computation, as this work focuses on efficient online deployment; notably, we do not consider~\citet{kleinegesse2021sequential}, as it performs online information gain estimation with heavy online computation. For all environments, we consider 1) a random baseline and 2) the MINEBED~\cite{pmlr-v119-kleinegesse20a} method using Bayesian Optimization (labeled MINEBED-BO),
a static BOED baseline applicable to non-differentiable implicit models. For applicable environments, we additionally include 3) iDAD~\cite{ivanova2021implicit} trained with the InfoNCE bound and 4) DAD~\cite{foster2021deep}. iDAD and DAD are not direct competitors to \rloed, they respectively require a differentiable simulator and analytic likelihoods. We nevertheless include them to measure the performance gap when such knowledge is available.

\para{Performance metrics} For the Location Finding task (Section~\ref{ssec:locfin}), we use the sequential Nested Monte Carlo upper bound on the total EIG given in~\citet{foster2021deep} as a performance metric. For the SIR and cartpole environments, we use the InfoNCE bound given in \citet{ivanova2021implicit} as a comparison metric. 

\subsection{Location Finding}\label{ssec:locfin}

We consider the Location Finding task used previously in \citet{ivanova2021implicit, foster2021deep}. In this environment, the unknown parameters $\theta$ are the locations of multiple hidden sources in an arbitrary (but given) $N$-dimensional space; for instance, with $N=2$, this is the Cartesian coordinate plane. The sources emit a signal whose intensity is inversely proportional to the squared distance from the source.
With multiple sources, the total intensity is the superimposition (or sum) of the individual signals. Each experiments measures a noisy signal at any specified location within given bounds to deduce the locations of the sources. 

We consider two hidden sources, $N \in \{2, 5, 10, 15\}$, and $T=10$ experiments to evaluate \rloed in both in high dimensions and compared to methods with stricter assumptions (iDAD, DAD) or general baselines (random, MINEBED-BO). We use a standard normal prior, $\theta\sim\mathcal{N}(0, 1)$. Table~\ref{table:locfin} summarizes the results.

We observe that \rloed significantly outperforms MINEBED-BO and the random baseline. While in lower dimensions, \rloed does not outperform DAD and iDAD (which make stronger assumptions), in higher dimensions ($N=10, 15$), \rloed becomes competitive with iDAD and approaches the performance of DAD for $N=15$. MINEBED-BO appears to perform worse in higher dimensions, even compared to the random baseline, which we hypothesize is due to Bayesian Optimization being less effective in higher dimensions~\cite{kirschner2019adaptive}.

In Table~\ref{table:locfin}, we observe a measurable performance gap between iDAD/DAD and \rloed, which we hypothesize is a limitation of deep RL. Given that the sPCE bound in~\eqref{eqn:spce} is tractable in this environment, we evaluated the performance of a policy trained using analytic likelihoods. In Figure~\ref{fig:reward}, we observe that even with analytic likelihoods, the RL agent still does not match the performance of iDAD or DAD, suggesting that the RL algorithm is the main limiting factor. As such, we still suggest using iDAD for diffentiable models or DAD for analytic models whenever such assumptions hold. Section~\ref{ssec:cartpole} examines a case where they do not hold.

\begin{table}[t]
\caption{Location finding task results (Section~\ref{ssec:locfin}). Upper bound of the expected information gain (EIG) of five BOED methods, with four $N$-dim. spaces. EIG values are estimated by the sNMC~\cite{foster2021deep} upper bound with $L = 1\times10^5$ contrastive samples to ensure tightness of the bound. Methods above the horizontal line make stronger assumptions than \rloed.} \label{table:locfin}
\vskip -0.1in
\begin{center}
\begin{small}
\begin{tabular}{lcccc}
\toprule
Method & $N=2$ & $N=5$ & $N=10$ & $N=15$\\
\midrule
DAD     & 7.967 &  3.337 & 0.937 & 0.431\\
iDAD (InfoNCE)   &  5.637 & 5.637 & 0.744 & 0.325\\
\midrule
Random    & 4.862 &  1.899 & 0.570 & 0.221\\
MINEBED-BO & 5.105 &  2.139 & 0.194 & 0.109\\
\textbf{\rloed }   & \textbf{7.100} &  \textbf{2.455} & \textbf{0.766} & \textbf{0.407}\\
\bottomrule
\end{tabular}
\end{small}
\end{center}
\vskip -0.1in
\end{table}

\subsubsection{Ablation Study: Sparse Rewards}\label{sssec:ablation-sparse}

We investigate the impact of the reward formulation on RL training by comparing \rloed's performance and convergence speed under both the dense and sparse rewards. 
While agents should converge to the same optimal policy under each of these two rewards---as both reward signals have the same sum over each trajectory---we empirically observe that agents trained with the sparse reward often fail to converge entirely, suggesting that the dense reward signal significantly improves learning. Figure~\ref{fig:reward} shows the mutual information lower bound evaluation curves corresponding to the sparse and dense rewards. We observe that following initial random exploration, the sparse reward agent collapses while the dense reward agent quickly converges to be competitive with the agent trained on the sPCE-based reward.

\begin{figure}[t]
    \centering
    \includegraphics[width=\columnwidth]{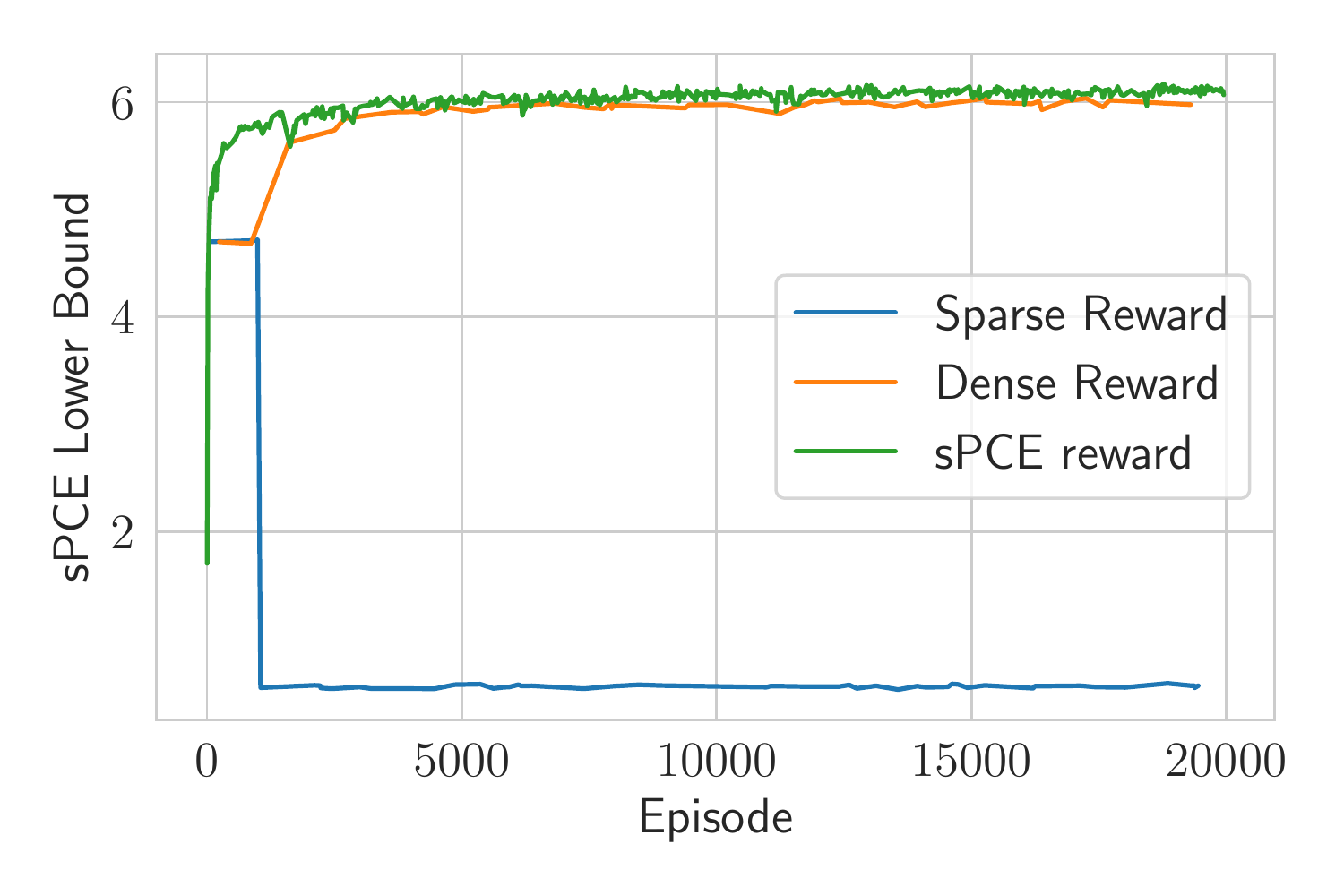}
    \vskip -0.2in
    \caption{Convergence of the sPCE lower bound with 4095 contrastive samples on the Location Finding task. The sparse reward formulation fails to learn, quickly collapsing after the initial random exploration phase. The dense reward quickly approaches the performance of the sPCE-based~\cite{foster2021deep} reward, despite the sPCE reward using analytic likelihoods.}
    \label{fig:reward}
    \vskip -0.1in
\end{figure}

\subsection{SIR Model}\label{ssec:sir}

We next consider the Susceptible-Infected-Recovered (SIR) model, a differentiable implicit model from epidemiology used previously in BOED \cite{ivanova2021implicit, kleinegesse2021gradient}. We follow the formulation from \citet{ivanova2021implicit}, who use a stochastic SIR formulation based on stochastic differential equations from \citet{kleinegesse2021gradient}. In the SIR model, there is a fixed-size population with three groups: susceptible, where individuals are infected according to a model parameter $\beta$; infected, where individuals recover according to a model parameter $\gamma$; and recovered. Thus, the unknown model parameters are $\theta:=[\beta, \gamma]$. The design space $\Xi$ consists of a time $\kappa$ at which to measure the number of infected individuals to estimate the unknown parameters $\theta$. Note that in this environment, experiments are not conditionally independent; thus, we use an LSTM-based history encoder.

\begin{table}[t]
\caption{SIR (Section~\ref{ssec:sir}) and Cartpole (Section~\ref{ssec:cartpole}) results: lower bound of the EIG of four BOED methods. EIG values are estimated by the InfoNCE bound~(\ref{eqn:nce}) with $1\times10^5$ contrastive samples to ensure tightness of the bound. However, as the InfoNCE bound relies on a learned critic, it may be biased and is not necessarily a guarantee of results. iDAD makes stronger assumptions (i.e., simulator differentiability) than \rloed.} \label{table:sir}
\vskip -0.1in
\begin{center}
\begin{small}
\begin{tabular}{lc|c}
\toprule
Method/Environment & SIR & Cartpole\\
\midrule
iDAD (InfoNCE)   & 3.843 & N/A\\
\midrule
Random    & 1.915 & 3.434\\
MINEBED-BO & 2.539 & 3.628\\
\textbf{\rloed}& \textbf{3.715} & \textbf{4.802}\\
\bottomrule
\end{tabular}
\end{small}
\end{center}
\vskip -0.2in
\end{table}

Table~\ref{table:sir} summarizes the results. We observe that \rloed outperforms both the random and MINEBED-BO baselines, and remains competitive with iDAD. We find that the posterior estimates are generally consistent with the ground truth. Figure~\ref{fig:sir_posterior} presents a sample posterior estimate with ground truth parameters $\theta=[0.924, 0.073]$. 

\begin{figure}[t]
    \centering
    \includegraphics[width=\columnwidth]{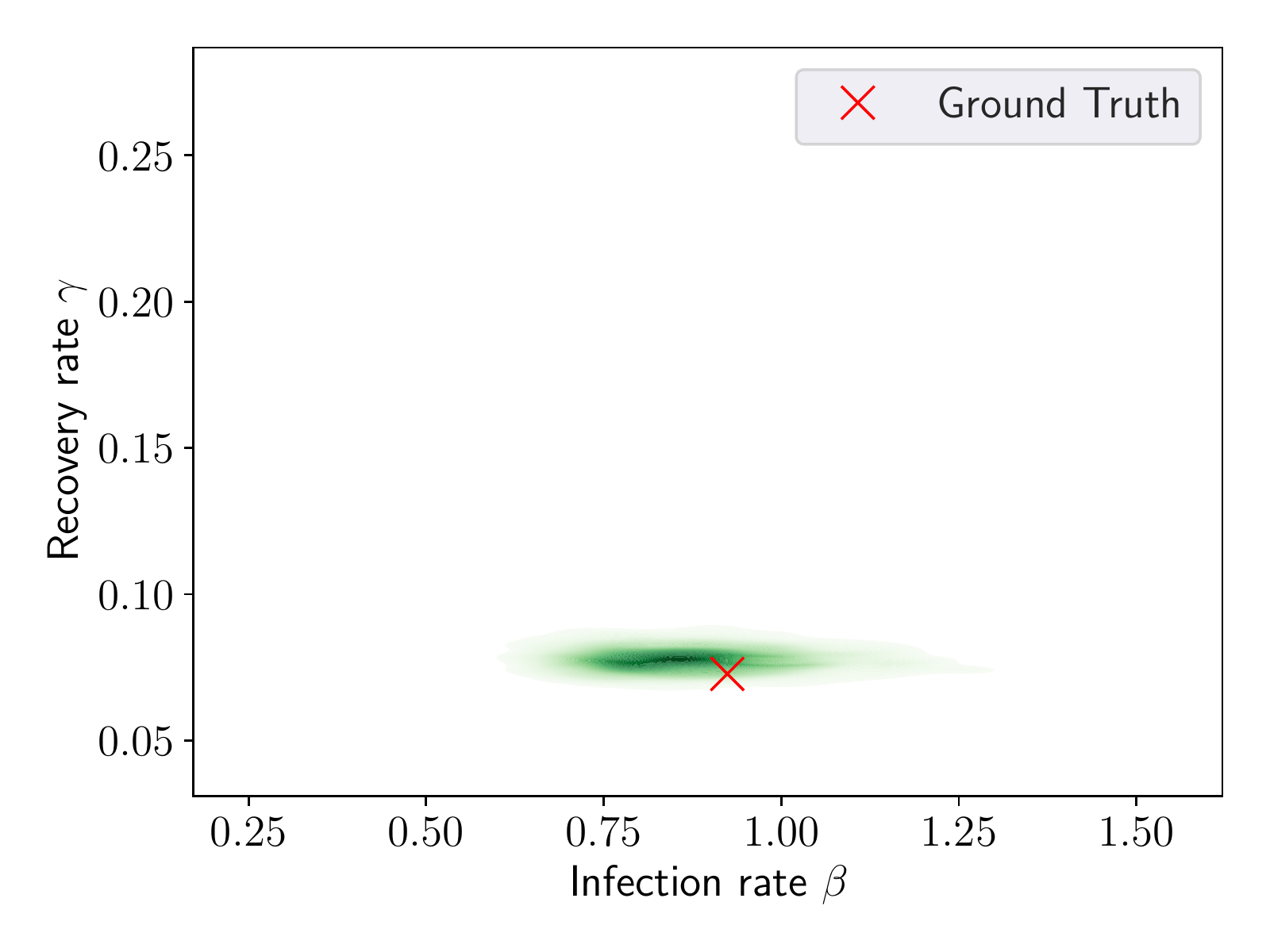}
    \vskip -0.2in
    \caption{Example posterior estimate using the critic $U_\psi$ on the SIR environment.}
    \label{fig:sir_posterior}
    \vskip -0.2in
\end{figure}

\subsection{Cartpole}\label{ssec:cartpole}

Finally, we consider variations of the standard cart pole environment~\cite{nagendra2017comparison}, as implemented by \citet{makoviychuk2021isaac} in NVIDIA Isaac Gym. The cartpole environment represents a nondifferentiable implicit model with limited observations and conditionally dependent experiments. The cartpole environment consists of a cart that slides horizontally on a bar, with a pole freely attached to it on one end. The unknown parameters $\theta=[\mu, m]$ consist of the (rotational) friction $\mu$ between the cart and pole and the pole's mass $m$. Each experiment applies an impulse to the cart every $1/6\,s=16.6$ ms. The agent does not observe the intermediate states of the cartpole between timesteps, and only observes the cartpole's state at the end of each $1/6\,s$ timestep. We set the maximum number of experiments to $T=5$ and the prior to be uniform over the bounds, i.e. $m \sim U[0.5, 1.5]$ and $\mu \sim U[0, 0.2]$. The observations consist of the position and velocity of the cart and pole's joint angle. 

Table~\ref{table:sir} presents results for the Cartpole environment. Note that iDAD cannot be used with this environment, as gradient paths are unavailable through the Isaac Gym simulator. Figure~\ref{fig:cartpole_posterior} shows a sample posterior estimate with ground truth $\theta=[0.037, 1.02]$. We find that posterior estimates are generally tight around the ground truth value, as represented by the shaded density around the ground truth marker.

\begin{figure}[t]
    \centering
    \includegraphics[width=\columnwidth]{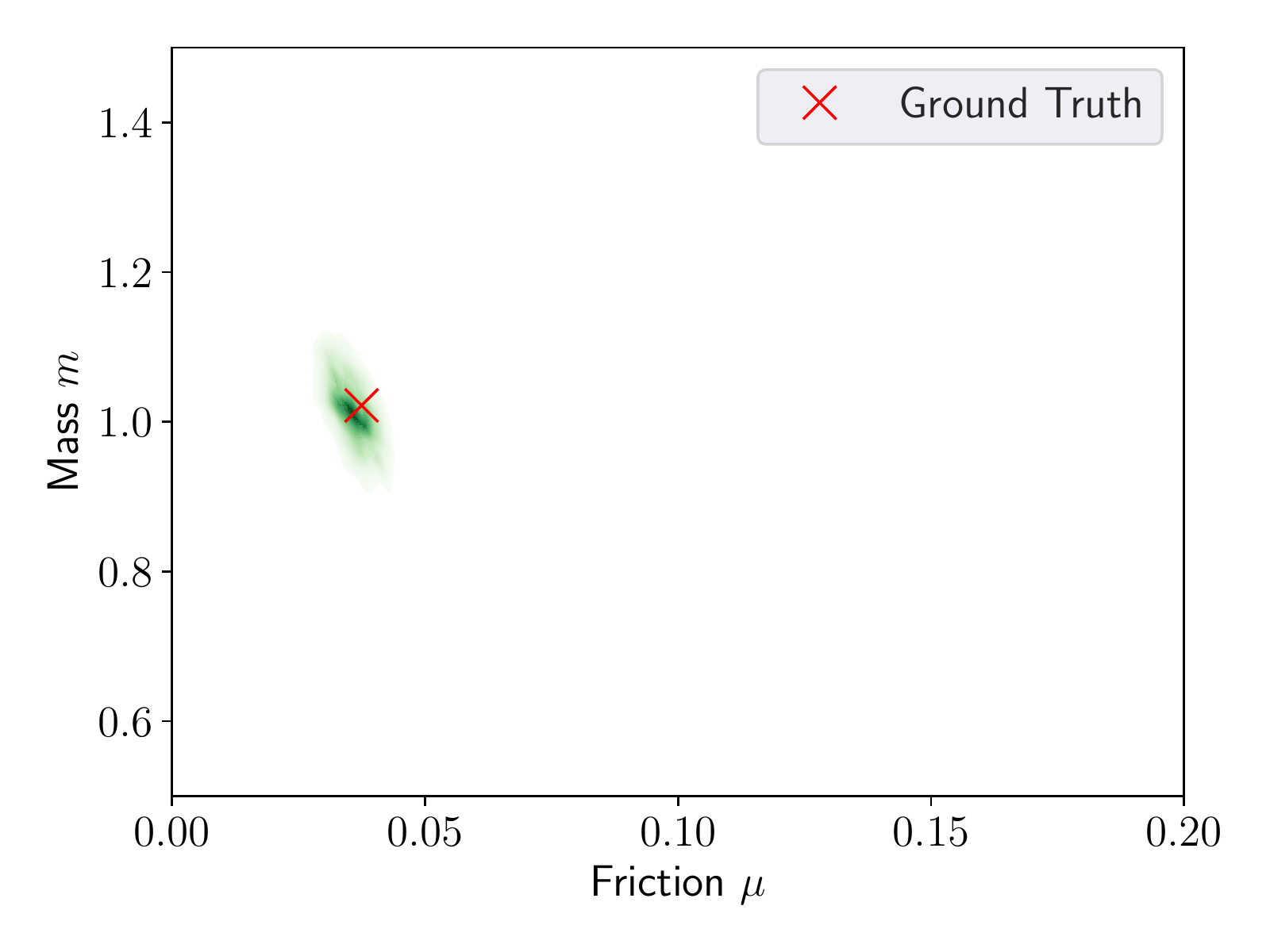}
    \vskip -0.2in
    \caption{Example posterior estimate using the critic $U_\psi$ on the Cartpole environment. }
    \label{fig:cartpole_posterior}
    \vskip -0.2in
\end{figure}

\section{Discussion}\label{ssec:discussion}

\para{Conclusion} This work proposes \rloed, a novel algorithm for Bayesian optimal experiment design (BOED), which leverages RL to train a design policy upfront, utilizing a simulator of the environment and a reward function based on the information gain. While prior work on policy-based BOED assumes a closed-form likelihood, conditionally-independent experiments, and/or a differentiable simulator, \rloed is a more general approach compatible with any black-box environment simulator. Simulation results suggest that \rloed yields promising performance on three benchmark tasks and is competitive with baselines.

\para{Limitations and future work} The primary limitation of this work is the initial training time required. As \rloed depends on deep RL, it also inherits its sample inefficiency and brittleness with respect to hyperparameters and implementation~\cite{engstrom2020implementation}. Future work will include applying \rloed to real-world tasks such as robot manipulation and Real2Sim transfer~\cite{lim2021planar}. We also hope to address the high computational cost during training and reduce the required number of simulated samples.

\bibliography{example_paper}
\bibliographystyle{icml2022}

\newpage

\appendix

\onecolumn
\section{Appendix: Mathematical Details}

\subsection{Decomposing the Total Expected Information Gain as a Sum of Marginals}\label{app:EIG_decomp}

In this section, we show Equation~\eqref{eqn:sum_marginals}:
\begin{equation*}
    I(\theta; y_1, \ldots, y_T \mid \xi_1, \ldots, \xi_T) = \sum_{t = 1}^T I_{h_{t - 1}}(\xi_t).
\end{equation*}

This holds via the chain rule for mutual information~\cite{cover2012elements}, since:
\begin{flalign*}
    I(\theta; y_1, \ldots, y_T \mid \xi_1, \ldots, \xi_T) &\overset{(a)}= \sum_{t = 1}^T I(\theta; y_t \mid \xi_1, \ldots, \xi_T, y_1, \ldots, y_{t - 1}) \\ &\overset{(b)}= \sum_{t = 1}^T I(\theta; y_t \mid \xi_1, \ldots, \xi_{t - 1}, \xi_t, y_1, \ldots, y_{t - 1})  \\ &= \sum_{t = 1}^T I(\theta; y_t \mid h_{t - 1}, \xi_t) \overset{(c)}= \sum_{t = 1}^T I_{h_{t - 1}}(\xi_t),
\end{flalign*}
where (a) invokes the chain rule for mutual information, (b) utilizes that $\theta$ and $y_t$ do not depend on $\xi_i, i > t$, and (c) applies the definition of $I_{h_{t - 1}}(\xi_t)$.

\subsection{Expected Total Information Gain for a Policy}\label{app:total_EIG}

Recall that we define the expected total information gain of a policy $\pi$ as follows:
\begin{equation*}
    \mathcal{I}_T(\pi) = \mathbb{E}_{p(\theta)p(h_T\mid\theta, \pi)}\left[\sum_{t=1}^T I_{h_{t-1}}(\xi_t)\right], \xi_t \sim \pi(h_{t-1}).
\end{equation*}

In this section, we adapt Proposition 1 from~\citet{ivanova2021implicit} for stochastic policies $\pi$. The updated result is stated and proven below:

\begin{proposition}[Total expected information gain of a policy]\label{prop:total_EIG}
Under the data generating distribution $p(h_T \mid \theta, \pi) = \prod_{t = 1:T} p(y_t \mid \theta, \xi_t, h_{t - 1})\pi(\xi_t \mid h_{t - 1})$, where $\xi_t \sim \pi(h_{t - 1})$ are the designs generated by the policy $\pi$, we can write $\mathcal{I}_T(\pi)$ as follows:
\begin{flalign}
    \mathcal{I}_T(\pi) &= \mathbb{E}_{p(\theta)p(h_T \mid \theta, \pi)}\left[\log p(h_T \mid \theta, \pi) \right] - \mathbb{E}_{p(h_T \mid \pi)}\left[\log p(h_T \mid \pi) \right] \label{prop1_line1} \\ &= I(\theta; h_T \mid \pi). \label{prop1_line2}
\end{flalign}
\end{proposition}

\textit{\underline{Remark 1}}: Note that Proposition 1 in~\citet{ivanova2021implicit} proves an equivalent result for deterministic policies, for which $\xi_t = \pi(h_{t-1})$, while this work leverages RL algorithms that utilize stochastic policies, for which $\xi_t \sim \pi(h_{t-1})$; otherwise, the two results are entirely analogous. The following proof for stochastic policies is very similar to that of Proposition 1 in~\citet{ivanova2021implicit}, but changes in a few step. We reproduce the altered proof below, though citing a result from~\citet{ivanova2021implicit} for a proof segment that remains unchanged.

\textit{\underline{Remark 2}}: While~\citet{ivanova2021implicit} mentions that $\mathcal{I}_T(\pi)$ is not a true information gain because the designs are deterministic rather than random variables, in our case $\mathcal{I}_T(\pi)$ can indeed be interpreted as a mutual information, as shown below.

\begin{proof}

Observe that:
\begin{flalign}
    \mathcal{I}_T(\pi) :&= \mathbb{E}_{p(\theta)p(h_T \mid \theta, \pi)}\left[\sum_{t = 1}^T I_{h_{t - 1}}(\xi_t) \right] \overset{(a)}= \sum_{t = 1}^T\mathbb{E}_{p(\theta)p(h_T \mid \theta, \pi)}\left[I_{h_{t - 1}}(\xi_t) \right] \nonumber \\ &\overset{(b)}= \sum_{t = 1}^T\mathbb{E}_{p(\theta)p(h_{t - 1}, \xi_t \mid \theta, \pi)}\left[I_{h_{t - 1}}(\xi_t) \right] = \sum_{t = 1}^T\mathbb{E}_{p(\theta)p(h_{t - 1} \mid \theta, \pi)p(\xi_t \mid h_{t - 1}, \pi})\left[I_{h_{t - 1}}(\xi_t) \right] \nonumber \\ &\overset{(c)}= \sum_{t = 1}^T\mathbb{E}_{p(h_{t - 1} \mid \pi)p(\theta \mid h_{t - 1})p(\xi_t \mid h_{t - 1}, \pi})\left[I_{h_{t - 1}}(\xi_t) \right], \label{eqn:ivanovna-14}
\end{flalign}
where (a) applies linearity of expectation, (b) uses that $I_{h_{t - 1}}(\xi_t)$ does not depend on $x_i$ for $i > t$ or on $y_i$ for $i > t - 1$, and (c) applies Bayes rule.

Next, we apply a result shown in Equations (15)-(20) in \citet{ivanova2021implicit}. The proof of this identity is unchanged from~\citet{ivanova2021implicit}, since it is only concerned with the expression $I_{h_{t - 1}}(\xi_t)$, which is conditioned on a fixed $\xi_t$ and therefore is unaffected by the $\xi_t$-generating process:
\begin{equation}\label{eqn:ivanovna-15-20}
    I_{h_{t - 1}}(\xi_t) = \mathbb{E}_{p(y_t \mid \xi_t, h_{t - 1})}\left[\mathbb{E}_{p(\theta \mid h_t)}[\log p(\theta \mid h_t)] - \mathbb{E}_{p(\theta \mid h_{t - 1})}[\log p(\theta \mid h_{t - 1})] \right].
\end{equation}

Substituting Equation~\eqref{eqn:ivanovna-15-20} into \eqref{eqn:ivanovna-14} yields:
\begin{flalign*}
    \mathcal{I}_T(\pi) &= \sum_{t = 1}^T\mathbb{E}_{p(h_{t - 1} \mid \pi)p(y_t \mid \xi_t, h_{t - 1})p(\xi_t \mid h_{t - 1}, \pi)}\left[\mathbb{E}_{p(\theta \mid h_t)}[\log p(\theta \mid h_t)] - \mathbb{E}_{p(\theta \mid h_{t - 1})}[\log p(\theta \mid h_{t - 1})]\right].
\end{flalign*}

Using that:
\begin{flalign*}
p(h_t \mid \pi) &= p(h_{t - 1}, \xi_t, y_t \mid \pi) = p(h_{t - 1} \mid \pi)p(\xi_t \mid \pi, h_{t - 1})p(y_t \mid \pi, h_{t - 1}, \xi_t) \\ &= p(h_{t - 1} \mid \pi)p(y_t \mid \xi_t, h_{t - 1})p(\xi_t \mid h_{t - 1}, \pi),
\end{flalign*}
we obtain:
\begin{flalign*}
    \mathcal{I}_T(\pi)&= \sum_{t = 1}^T\mathbb{E}_{p(h_t \mid \pi)}\left[\mathbb{E}_{p(\theta \mid h_t)}[\log p(\theta \mid h_t)] - \mathbb{E}_{p(\theta \mid h_{t - 1})}[\log p(\theta \mid h_{t - 1})]\right] \\ &\overset{(a)}= \mathbb{E}_{p(h_T \mid \pi)} \sum_{t = 1}^T \left[\mathbb{E}_{p(\theta \mid h_t)}[\log p(\theta \mid h_t)] - \mathbb{E}_{p(\theta \mid h_{t - 1})}[\log p(\theta \mid h_{t - 1})]\right] \\ &\overset{(b)}= \mathbb{E}_{p(h_T \mid \pi)} \left[\mathbb{E}_{p(\theta \mid h_T)}[\log p(\theta \mid h_T)] - \mathbb{E}_{p(\theta)}[\log p(\theta)]\right],
\end{flalign*}
where (a) utilizes that $p(\theta \mid h_t)$ does not depend on $\xi_i$ or $y_i$ for $i > t$, and (b) collapses a telescoping sum. Finally, applying Bayes rule:
\begin{flalign*}
    \mathcal{I}_T(\pi) &= \mathbb{E}_{p(h_T \mid \pi)p(\theta \mid h_T)} \left[\log p(\theta \mid h_T)] - \mathbb{E}_{p(\theta)}[\log p(\theta)]\right] \\ &= \mathbb{E}_{p(h_T \mid \pi)p(\theta \mid h_T)} \left[\log p(\theta \mid h_T)] - \log p(\theta)\right] \\ &= \mathbb{E}_{p(h_T, \theta \mid \pi)} \left[\log p(\theta \mid h_T)] - \log p(\theta)\right] \\ &= \mathbb{E}_{p(\theta)p(h_T \mid \theta, \pi)} \left[\log p(h_T \mid \theta, \pi)] - \log p(h_T \mid \pi)\right],
\end{flalign*}
proving~\eqref{prop1_line1}.

To show~\eqref{prop1_line2}, we apply the definition of mutual information, followed by Bayes rule:
\begin{flalign*}
    I(\theta; h_T \mid \pi) &= \mathbb{E}_{p(\theta; h_T \mid \pi)}\left[\log\frac{p(\theta; h_T \mid \pi)}{p(\theta)p(h_T \mid \pi)} \right] = \mathbb{E}_{p(\theta)p(h_T \mid \theta, \pi)}\left[\log\frac{p(\theta)p(h_T \mid \theta, \pi)}{p(\theta)p(h_T \mid \pi)} \right] \\ &= \mathbb{E}_{p(\theta)p(h_T \mid \theta, \pi)}\left[\log\frac{p(h_T \mid \theta, \pi)}{p(h_T \mid \pi)} \right] = \mathbb{E}_{p(\theta)p(h_T \mid \theta, \pi)}\left[\log p(h_T \mid \theta, \pi) - \log p(h_T \mid \pi) \right].
\end{flalign*}

\end{proof}

\subsection{InfoNCE Bound for Stochastic Design Policies}\label{app:NCE}

\citet{ivanova2021implicit} shows (in Proposition 3 therein) that the InfoNCE bound lower-bounds the total information gain $\mathcal{I}_T(\pi)$ for deterministic policies. Furthermore,~\citet{ivanova2021implicit} shows that the optimal critic (with arbitrary approximating capacity) recovers the sPCE lower bound to the information gain, which is applicable with closed-form likelihoods, and that the inequality is tight as the number of contrastive samples $L$ approaches $\infty$.

We restate Proposition 3 from~\citet{ivanova2021implicit} below:
\begin{proposition}[InfoNCE bound for implicit policy-based BOED]\label{prop:NCE}
    Let $\theta_{1:L} \sim p(\theta_{1:L}) = \prod_i p(\theta_i)$, where $L \ge 1$, be a set of contrastive samples. For design policy $\pi$ and critic function $U: \mathcal{H}^T \times \Theta \to \mathbb{R}$, let:
    \begin{equation*}
        \mathcal{L}_T^{NCE}(\pi, U; L) := \mathbb{E}_{p(\theta_0)p(h_T \mid \theta_0, \pi)}\mathbb{E}_{p(\theta_{1:L})}\left[\log \frac{\exp(U(h_T, \theta_0))}{\frac{1}{L + 1}\sum_{i=0}^L \exp(U(h_T, \theta_i))} \right].
    \end{equation*}
    
    Then, $\mathcal{I}_T(\pi) \ge \mathcal{L}_T^{NCE}(\pi, U; L)$ for any $U$ and $L \ge 1$. Furthermore, the optimal critic $U^*_{NCE}(h_T, \theta) = \log p(h_T \mid \theta, \pi) + c(h_T)$, where $c(h_T)$ is any arbitrary function depending only on the history, recovers the sPCE information gain bound in~\citet{ivanova2021implicit}. The inequality is tight in the limit as $l \to \infty$ for this optimal critic.
\end{proposition}

While~\citet{ivanova2021implicit} assume deterministic policies, their proof of Proposition 3 only utilizes policy determinism by invoking Proposition 1 from~\citet{ivanova2021implicit}. Since we already showed how to adapt Proposition 1 from~\citet{ivanova2021implicit} for stochastic policies, Proposition 3 applies to stochastic (as well as deterministic) policies without any further adaptation needed.

\subsection{Showing the Markovian Property for the MDP Underlying the BOED POMDP}\label{app:markov}

Consider the POMDP defined for the BOED problem in Section~\ref{ssec:boed-mdp}. In this section, we show that conditioned on this POMDP's underlying states, the transition dynamics and rewards are both Markovian.

To be Markovian, the states, actions, and rewards of the underlying MDP must satisfy the following property~\cite{sutton2018reinforcement}:
\begin{equation}
    p(s_{t + 1}, r_{t + 1} \mid s_0, \ldots, s_t, a_0, \ldots, a_t) = p(s_{t + 1}, r_{t + 1} \mid s_t, a_t).
\end{equation}

This holds because the current state, $s_t = h_t$, includes full information about all previous states $h_i$ for $i < t$, as well as all actions $a_i = \xi_{i + 1}$ for $i < t$:
\begin{flalign*}
    p(s_{t + 1}, r_{t + 1} \mid s_0, \ldots, s_t, a_0, \ldots, a_t) &= p(h_{t+1}, \theta_{0:L}, r_{t + 1} \mid h_{0:t}, \xi_{1:t+1}, \theta_{0:L}) \\ &= p(h_{t+1}, \theta_{0:L}, r_{t + 1} \mid h_t, \xi_{t+1}, \theta_{0:L}) \\ &= p(s_{t + 1}, r_{t + 1} \mid s_t, a_t).
\end{flalign*}

\subsection{The Dense Reward Signal Approximates the Marginal Expected Information Gain Values}\label{app:dense_reward}

Recall that the dense reward is defined as follows:
\begin{equation}\label{eqn:dense}
    r(s_t, a_t) = g(h_{t + 1}, U_\psi; L) - g(h_t, U_\psi; L).
\end{equation}

Firstly, as noted in Section~\ref{ssec:dense-reward}, the sum of dense rewards over a given trajectory is equal to the total reward under the sparse reward formulation:
\begin{equation}\label{eqn:reward_sum}
    \sum_{t=0}^{T-1} r(s_t, a_t) = \sum_{t=1}^T \left[g(h_t, U_\psi; L)-g(h_{t-1}, U_\psi; L)\right] = g(h_T, U_\psi; L),
\end{equation}
where we collapse the telescoping sum and use that $g(h_0, U_\psi; L):=0$, which holds by definition. The value $g(h_T, U_\psi; L)$ is the reward assigned in the last step of the episode under the sparse reward.

Next, we aim to interpret the individual rewards $r(s_t, a_t)$. By Proposition~\ref{prop:total_EIG} and Proposition~\ref{prop:NCE}, we have the following relationship between the InfoNCE bound and the mutual information:
\begin{equation}\label{eqn:NCE_IG}
    \mathbb{E}_{p(\theta_{0:L})p(h_T \mid \theta_0, \pi)}[g(h_T, U_\psi; L)] = \mathcal{L}_T^{NCE}(\pi, U_\psi; L) \le \mathcal{I}_T(\pi) = I(\theta; h_T \mid \pi),
\end{equation}
where the inequality is tight as $U_\psi$ approaches the optimal critic and $L \to \infty$.

We argue that in the limit as $U_\psi$ approaches the optimal critic and $L \to \infty$, the dense rewards $r(s_t, a_t)$ in~\eqref{eqn:dense} approach the marginal information gains in expectation:
\begin{proposition}
    Let $r(s_t, a_t) := g(h_{t + 1}, U_\psi; L) - g(h_t, u_\psi; L)$. Then, in the limit as $U_\psi$ approaches the optimal critic and $L \to \infty$, $\mathbb{E}_{p(\theta_{0:L})p(h_{t+1} \mid \pi, \theta_0)}[r(s_t, a_t)] \to I(\theta; h_{t+1} \mid \pi) - I(\theta; h_t \mid \pi) = I(\theta; h_{t+1} \mid \pi, h_t)$, which is the marginal expected information gain from time $t$ to time $t+1$.
\end{proposition}

\begin{proof}
We prove the result by induction. First, consider the base case in which $T=1$. In this case,~\eqref{eqn:NCE_IG} implies that:
\begin{flalign*}
     \mathbb{E}_{p(\theta_{0:L})p(h_1 \mid \theta_0, \pi)}[r(s_0, a_0)] &\overset{(a)}= \mathbb{E}_{p(\theta_{0:L})p(h_1 \mid \theta_0, \pi)}[g(h_1, U_\psi; L)-g(h_{0}, U_\psi; L)] \\ &\overset{(b)}= \mathbb{E}_{p(\theta_{0:L})p(h_1 \mid \theta_0, \pi)}[g(h_1, U_\psi; L)] \overset{(c)}\le I(\theta; h_1 \mid \pi),
\end{flalign*}
where (a) uses the definition of the reward, (b) uses that $g(h_0, u_\psi; L) = 0$, and (c) applies~\eqref{eqn:NCE_IG}. Since the inequality is tight in the limit as $U_\psi$ approaches the optimal critic and $L \to \infty$:
\begin{equation*}
    \mathbb{E}_{p(\theta_{0:L})p(h_1 \mid \theta_0, \pi)}[r(s_0, a_0)] \to I(\theta; h_1 \mid \pi) = I(\theta; h_1 \mid \pi, h_0),
\end{equation*}
in this limit.

Next, for $T=2$, we apply~\eqref{eqn:reward_sum} and ~\eqref{eqn:NCE_IG} to obtain:
\begin{flalign*}
    \mathbb{E}_{p(\theta_{0:L})p(h_2 \mid \theta_0, \pi)}[r(s_0, a_0) + r(s_1, a_1)] &= \mathbb{E}_{p(\theta_{0:L})p(h_1 \mid \theta_0, \pi)}[r(s_0, a_0)] + \mathbb{E}_{p(\theta_{0:L})p(h_2 \mid \theta_0, \pi)}[r(s_1, a_1)] \\ &= \mathbb{E}_{p(\theta_{0:L})p(h_2 \mid \theta_0, \pi)}[g(h_2, U_\psi; L)] \le I(\theta; h_2 \mid \pi).
\end{flalign*}
In the limit as $U_\psi$ approaches the optimal critic and $L \to \infty$, the inequality is tight, and furthermore, invoking the $T=1$ case above, the expectation of $r(s_0, a_0)$ approaches $I(\theta; h_1 \mid \pi, h_0)$. Therefore, in the limit:
\begin{equation*}
    \mathbb{E}_{p(\theta_{0:L})p(h_2 \mid \theta_0, \pi)}[r(s_1, a_1)] \to I(\theta; h_2 \mid \pi) - I(\theta; h_1 \mid \pi) \overset{(a)}= I(\theta; h_2 \mid \pi, h_1),
\end{equation*}
where (a) holds by the chain rule for mutual information.

Next, we assume that for some $T = k$, $\mathbb{E}_{p(\theta_{0:L})p(h_{t+1} \mid \theta_0, \pi)}[r(s_t, a_t)]$ approaches $I(\theta; h_{t+1} \mid \pi, h_t)$ in the limit for each $t \in \{0, \ldots, k-2\}$.

By~\eqref{eqn:reward_sum} and ~\eqref{eqn:NCE_IG}:
\begin{equation*}
    \mathbb{E}_{p(\theta_{0:L})p(h_k \mid \theta_0, \pi)}\left[\sum_{t=0}^{k-1} r(s_t, a_t)\right] = \mathbb{E}_{p(\theta_{0:L})p(h_k \mid \theta_0, \pi)}[g(h_k, U_\psi; L)] \le I(\theta; h_k \mid \pi),
\end{equation*}
where the inequality is tight in the limit.

By the induction hypothesis, the left-hand side approaches the following in the limit:
\begin{flalign}
    \mathbb{E}_{p(\theta_{0:L})p(h_k \mid \theta_0, \pi)}\left[\sum_{t=0}^{k-1} r(s_t, a_t)\right] &= \sum_{t=0}^{k-1} \mathbb{E}_{p(\theta_{0:L})p(h_{t+1} \mid \theta_0, \pi)}\left[r(s_t, a_t)\right] \nonumber \\ &\to \sum_{t=0}^{k-2} I(\theta; h_{t+1} \mid \pi, h_t) + \mathbb{E}_{p(\theta_{0:L})p(h_k \mid \theta_0, \pi)}[r(s_{k-1}, a_{k-1})]. \label{eqn:induction_sum}
\end{flalign}

Thus, the expression in~\eqref{eqn:induction_sum} must be equal to $I(\theta; h_k \mid \pi)$ in the limit, and so:
\begin{flalign*}
    \mathbb{E}_{p(\theta_{0:L})p(h_k \mid \theta_0, \pi)}[r(s_{k-1}, a_{k-1})] &\to I(\theta; h_k \mid \pi) - \sum_{t=0}^{k-2} I(\theta; h_{t+1} \mid \pi, h_t) \\ &\overset{(a)}= I(\theta; h_k \mid \pi) - I(\theta; h_{k-1} \mid \pi) \overset{(b)}= I(\theta; h_k \mid \pi, h_{k-1}),
\end{flalign*}
in the limit as $U_\psi$ approaches the optimal critic and $L \to \infty$, as desired, where (a) and (b) both hold by the chain rule for mutual information.

\end{proof}

\section{Experiment Details}
For all experiments, the policy and $Q$-networks are both 2 layers each with a hidden dimension of 256. To stabilize training, we borrow from RL literature learn two critic networks $U_\psi$ and $U_\psi^{\rm target}$. The critic $U_\psi$ is trained on batches from the current rollout, and the ``target" critic $U_\psi^{\rm target}$ is updated with $\tau=0.005$. The reward is calculated using the target critic $U_\psi^{\rm target}$. We observe that this provided a far more stable learning rate and reduced policy collapse. Finally, we remark that while the sparse and dense reward formulations are sum to the same value with the discount factor $\gamma=1$, we find that setting $\gamma<1$ provided for more stable training.

\begin{table}[h]
    \centering
    \caption{Training Hyperparameters}
    \begin{tabular}{|c|c|c|c|}
        \hline
        Parameter & Location Finding & SIR & Cartpole \\
        \hline
        \hline
        Learning Rate & $3\times 10^{-4}$ & $3\times 10^{-4}$ & $3\times 10^{-4}$ \\
        Initial Random Timesteps & $10^{4}$ & $10^{4}$ & $10^{4}$ \\
        Batch Size & $256$ & $256$ & $256$ \\
        Hidden Layer Size & $256$ & $256$ & $256$ \\
        \# Hidden Layers & $2$ & $2$ & $2$ \\
        \# Updates Per Timestep & $10$ & $8$ & $8$ \\
        Policy Update Frequency & 2 & 2 & 2 \\
        Policy Update Noise & 0.2 & 0.2 & 0.2 \\
        Exploration noise & 0.1 & 0.1 & 0.1 \\
        $\gamma$ & $0.99$ & $0.99$ & $0.99$ \\
        $\tau$ & $5 \times 10^{-3}$ & $5\times 10^{-3}$ & $5\times 10^{-3}$ \\
        Policy Noise & 0.2 & 0.2 & 0.2 \\
        \# Parallel Environments & 4096 & 512 & 2048 \\
        Critic History Encoder $E_{\psi_h}(h_t)$ & Attention & LSTM & LSTM \\
        Critic Learning Rate & $[1\times10^{-4}, 5\times10^{-4}]$ & $3\times10^{-4}$ & $3\times10^{-4}$ \\
        \hline
    \end{tabular}
    \label{tab:hyperparams_td3}
\end{table}

\end{document}